\documentclass[conference]{IEEEtran}
\IEEEoverridecommandlockouts
\usepackage{amsmath,amssymb,amsfonts}
\usepackage{algorithmic}
\usepackage{graphicx}
\usepackage{textcomp}
\usepackage{xcolor}

\def\BibTeX{{\rm B\kern-.05em{\sc i\kern-.025em b}\kern-.08em
    T\kern-.1667em\lower.7ex\hbox{E}\kern-.125emX}}
\usepackage{hyperref}
\usepackage{amsthm}

\usepackage[capitalize,noabbrev]{cleveref}
\usepackage[inline]{enumitem}
\usepackage{enumerate}
\usepackage{algorithm}
\usepackage{subfigure}

\usepackage{mathtools}
\usepackage{fancyhdr}
\theoremstyle{plain}
\newtheorem{theorem}{Theorem}[section]
\newtheorem{lemma}[theorem]{Lemma}

\theoremstyle{definition}

\theoremstyle{remark}

\usepackage{amsmath,amsfonts,bm}

\newcommand{\vect}[1]{\mathbf{#1}}

\newcommand{\normal}[0]{\mathcal{N}}

\DeclareMathOperator{\Prob}{\mathbb{P}}

\newcommand{\classifier}[1][]{f\ifx&#1&\else(#1)\fi}

\newcommand{\logit}[1][]{l\ifx&#1&\else(#1)\fi}

\newcommand{\classpred}[1][]{\eta\ifx&#1&\else(#1)\fi}

\newcommand{\ket}[1]{\lvert #1 \rangle}
\newcommand{\bra}[1]{\langle #1 \rvert }
\newcommand{\braket}[2]{\langle #1 \rvert #2 \rangle}

\counterwithin*{theorem}{section}
\newtheorem{innercustomgeneric}{\customgenericname}
\providecommand{\customgenericname}{}
\newcommand{\newcustomtheorem}[2]{%
  \newenvironment{#1}[1]
  {%
   \renewcommand\customgenericname{#2}%
   \renewcommand\theinnercustomgeneric{##1}%
   \innercustomgeneric
  }
  {\endinnercustomgeneric}
}

\newcustomtheorem{customthm}{Theorem}
\newcustomtheorem{customlemma}{Lemma}
\newcustomtheorem{customproposition}{Proposition}

\def\eqref#1{equation~\ref{#1}}

\def\1{\bm{1}}

\DeclareMathAlphabet{\mathsfit}{\encodingdefault}{\sfdefault}{m}{sl}
\SetMathAlphabet{\mathsfit}{bold}{\encodingdefault}{\sfdefault}{bx}{n}

\DeclarePairedDelimiterX{\infdivx}[2]{[}{]}{%
  #1\;\delimsize\|\;#2%
}

\usepackage[natbib=true, style=ieee, backend=biber, maxbibnames=9, maxcitenames=2, mincitenames=1, sorting=none]{biblatex}

\begin{document}

\title{Discrete Randomized Smoothing Meets\\ Quantum Computing
\thanks{The project/research is supported by the Bavarian Ministry of Economic Affairs, Regional Development and Energy with funds from the Hightech Agenda Bayern.}
}

\author{
    \IEEEauthorblockN{Tom Wollschl\"ager\IEEEauthorrefmark{1}\IEEEauthorrefmark{3}, Aman Saxena\IEEEauthorrefmark{1}\IEEEauthorrefmark{3}, Nicola Franco\IEEEauthorrefmark{2}, Jeanette Miriam Lorenz\IEEEauthorrefmark{2}, Stephan G\"unnemann\IEEEauthorrefmark{3}}
    \IEEEauthorblockA{\IEEEauthorrefmark{3}School of Computation, Information \& Technology, Technical Univ. of Munich, Germany
    \\\{t.wollschlaeger, a.saxena, s.guennemann\}@tum.de}
    \IEEEauthorblockA{\IEEEauthorrefmark{2}Fraunhofer Institute for Cognitive Systems IKS, Munich, Germany
    \\\{nicola.franco, jeanette.miriam.lorenz\}@iks.fraunhofer.de}
    \IEEEauthorblockA{\IEEEauthorrefmark{1} denotes equal contribution}
}

\fancypagestyle{specialfooter}{%
  \fancyhf{}
  \renewcommand\headrulewidth{0pt}
  \fancyfoot[R]{ \noindent\fbox{%
    \parbox{\textwidth}{%
        {\footnotesize \copyright 2024 IEEE. Personal use of this material is permitted. Permission from IEEE must be obtained for all other uses, in any current or future media, including reprinting/republishing this material for advertising or promotional purposes, creating new collective works, for resale or redistribution to servers or lists, or reuse of any copyrighted component of this work in other works.}
        }
    }}
}

\maketitle
\thispagestyle{plain}
\pagestyle{plain}
\thispagestyle{specialfooter}

\begin{abstract}
Breakthroughs in machine learning (ML) and advances in quantum computing (QC) drive the interdisciplinary field of quantum machine learning to new levels. However, due to the susceptibility of ML models to adversarial attacks, practical use raises safety-critical concerns. Existing Randomized Smoothing (RS) certification methods for classical machine learning models are computationally intensive. In this paper, we propose the combination of QC and the concept of discrete randomized smoothing to speed up the stochastic certification of ML models for discrete data.  
We show how to encode all the perturbations of the input binary data in superposition and use Quantum Amplitude Estimation (QAE) to obtain a quadratic reduction in the number of calls to the model that are required compared to traditional randomized smoothing techniques. In addition, we propose a new binary threat model to allow for an extensive evaluation of our approach on images, graphs, and text. 
\end{abstract}

\begin{IEEEkeywords}
Quantum Machine Learning, Certifiable Robustness, Randomized Smoothing, Quantum Amplitude Estimation.
\end{IEEEkeywords}

\section{Introduction}
\label{introduction-motivation}
With machine learning (ML) excelling in various domains such as natural language processing~\cite{brown2020language, thoppilan2022lamda, openai2023gpt4}, computer vision~\cite{han2022survey, dhariwal2021diffusion}, and health care~\cite{rajkomar2018, esteva2019} the need for reliable models emerged. Many approaches are vulnerable to attacks--noise crafted by an adversary to fool the model. Subsequently, the study of adversarial robustness deals with quantifying the robustness of the models and producing more robust approaches in these domains \cite{gnanasambandam2021optical, geisler2024attacking}. One method of quantification is the robustness certificate. There, the so-called certified radius is computed for an input sample in which the classifier's prediction remains unchanged~\cite{katz2017reluplex}. To this end, one has to provide a provable certificate that no noisy addition of lower magnitude can fool the classifier. Hence, existing certification methods are often computationally intensive. 
Lately, stochastic certification approaches have been dominating the field of adversarial robustness~\cite{cohen2019}. A promising method for discrete data certification is discrete randomized smoothing \cite{bojchevski2023efficient}, which provides robustness guarantees against adversarial perturbations for black-box classifiers. The certificate is obtained by building a smooth classifier using a sparsity-aware distribution. However, it involves computing the model predictions at exponentially growing input data perturbations, rendering the smooth classifier's exact evaluation computationally intractable. Therefore, in practice, Monte Carlo (MC) based stochastic methods are employed to evaluate the probabilistic approximation of the smooth classifier. 

Quantum Computing (QC), on the other hand, has shown promise in providing speedups for specific tasks, such as factoring large numbers \cite{Shor_1997} and unstructured database search \cite{grover1996fast}, making it a compelling approach to explore certifiable robustness. 
Consequently, research interest has surged at the intersection of QC and adversarial robustness, focusing on two primary categories. The first is to use QC to analyze and improve the adversarial robustness of classical neural networks~\cite{franco2022quantum, franco2023efficient}, and the second is to examine the robustness of quantum machine learning models~\cite{du2021quantum, gong2022enhancing, robustencodingsOurs, rotaionNoise}.

This work focuses on stochastic certification and addresses the computational challenges associated with robustness certification for machine learning models on discrete data. 
We propose a QC-based certification that requires quadratically fewer calls to the base classifier than the classical MC method to achieve similar accuracy in evaluating the smooth classifier. We achieve this by reformulating the estimation of the discrete smooth classifier as counting the weighted solution to a search problem. We frame the search problem as finding the perturbations for which the classifier predicts the desired class, thus relating it to the Quantum Amplitude Estimation (QAE) algorithm \cite{Brassard2000QuantumAA}. Furthermore, we propose a new threat model for our approach and can certify any data representation against discrete perturbations. Instead of certifying models that operate on binary data, we can work on arbitrary data with discrete perturbations; i.e., attacks can perturb data to finitely many possible states. This enables us to evaluate the utility of our method in certifying machine learning algorithms for \textit{natural language processing}, \textit{image classification}, and \textit{graph substructure counting} against discrete attacks that were previously out of reach for quantum computing certification.

Our core contributions can be summarized as follows:\footnote{Find our code at: \href{https://www.cs.cit.tum.de/daml/quantum-randomized-smoothing}{cs.cit.tum.de/daml/quantum-randomized-smoothing}.}
\begin{itemize}
    \item We develop a quantum algorithm to evaluate the smooth classifier for discrete data, requiring quadratically fewer calls to the model.
    \item We develop an extension to the sparsity-aware discrete randomized smoothing framework to include a larger class of data representations.
    \item We demonstrate the effectiveness of the proposed method both theoretically and empirically, highlighting its potential impact on the robustness of quantum machine learning algorithms.
\end{itemize}

\section{Related Work}
\label{related_work}
\textbf{Randomized Smoothing} has become one of the most prominent concepts of certifiable robustness in machine learning. Introduced by \citet{cohen2019}, originating from differential privacy \cite{YANG2024103827}, its applicability to any black-box model and strong certification performance sparked many variants. Examples include the use of confidence scores \cite{kumar2020certifying}, 
or localized smoothing improving robustness certificates for multi-output models \cite{schuchardt2023localized}. \citet{lee2020tight} introduce discrete smoothing and \citet{bojchevski2023efficient} extend the concept to sparse settings.

\textbf{QC and Adversarial Robustness} Recently, the intersection of QC and adversarial robustness has attracted considerable attention and has become a compelling area of research in the scientific community. 
To accelerate robustness verification for classic neural networks \citet{franco2022quantum, franco2023efficient}, decompose a mixed-integer program into a quadratic unconstrained optimization problem and use it to determine whether the prediction changes within a given adversarial budget. Other works directly investigate the robustness of QML models \cite{quantum_noise_protects, QHT_robustness, randomized_encodings}. Specifically, \citet{QHT_robustness} used the quantum hypothesis testing (QHT) framework to derive certification guarantees for QML models. On the other hand, \citet{quantum_noise_protects} developed a framework in which they add quantum noise to show differential privacy and eventually establish robustness guarantees. 

\textbf{QC and Monte Carlo} \citet{miyamoto2021bermudan} demonstrated the effectiveness of QAE as a quadratically faster alternative to MC estimation and has shown its application in finance. Based on this, \citet{Sahdev2022} proposed the application of QAE to evaluate the continuous smooth classifier, as defined in \cref{eq:rsdef_app}, and reiterated quadratic speed-up. However, any estimate of the smooth classifier must serve as a guaranteed lower bound to the actual value. In the context of continuous randomized smoothing, estimating the approximation of the actual expectation invariably involves some form of discretization. Although \citet{Sahdev2022} achieved a lower bound to the approximation, there was no guarantee that it remained a lower bound to the exact value. In addition, the accuracy of the estimate depends on the chosen discretization~\cite{franco2024quadratic}. Contrary, in this work, we avoid the need to analyze discretization errors by introducing a discrete framework.

\section{Background}
\label{background}
In this section, we introduce topics from the domain of adversarial robustness and quantum computing that are crucial for our discussion henceforth. We provide only a basic overview of these topics. %

\subsection{Certifiable Robustness}
\label{certifiable_robustness}

\textbf{Notation}
We denote the $N$ dimensional Euclidean space $\smash{\mathbb R ^{N}}$ as $\smash{\mathcal X^{(N)}}$ and the binary data space as $\smash{\mathcal X_D^{(N)}}$, that is, $\smash{\mathcal X_D^{(N)} := \{0,1\}^N}$. To simplify notation, we omit $(N)$ where possible. We express the binary classification task as a two-step process: we first map the input to the logit space using $\smash{\logit: \mathcal X \mapsto [0,1]}$, followed by mapping the logits to the class prediction via $\smash{\classpred: [0,1] \mapsto \{0,1\}}$, e.g., $\smash{\classpred(l):= \mathbf 1[l > 0.5]}$.
The binary classifier $\smash{\classifier: \mathcal X \mapsto \{0,1\}}$ is defined as $\smash{\classifier := \classpred \circ \logit}$.

\textbf{Problem Definition}
The research area of certifiable (provable) robustness aims to establish formal guarantees that---under all admissible attacks within the budget that does not change the semantic meaning of the data---none will alter the model's prediction. 
However, the process of examining the semantic similarity of the data is difficult\footnote{Identifying two semantically similar objects is the classification task.} and task-dependent \cite{gosch2023revisiting}. In the image domain, this difficulty is addressed by assuming that two inputs are semantically similar if they have an $L_p$ norm less than  $\epsilon$. Therefore, we can write that a model is robust for prediction $\vect x$ and attack budget $\epsilon$ if: 
\begin{equation}
\label{eq:2:certi}
f(\vect {\tilde x}) = f(\vect x) \text{ , } \forall \vect{\tilde x} \in \mathcal X\text{, s.t. }\lvert\lvert \vect{\tilde x} - \vect{x} \rvert\rvert_p \leq \epsilon  
\end{equation}
Robustness verification for neural networks can be formulated as a mixed-integer linear program (MILP) based on the neural network's weights~\cite{tjeng2019evaluating}. Unfortunately, the verification of properties in neural networks with ReLU activation is shown to be NP complete \cite{katz2017reluplex}. To avoid the computational complexity, approximations can be employed \cite{gowal2019}, which, however, potentially leads to loose guarantees. That is, it might not certify perturbation strengths that do not affect the model predictions. Randomized smoothing \cite{cohen2019} has gained significant attention among relaxed certification methods due to its ability to assess robustness without requiring knowledge of the classifier. %

\subsection{Randomized Smoothing}
Randomized smoothing as proposed by \citet{cohen2019} constructs a smooth classifier, denoted as $g$, evaluating the original classifier $\classifier$ in the vicinity of the input:
\begin{equation}
    \label{eq:rsdef_app}
    \begin{aligned}
        g(\vect x) = \mathbb E_{\vect z \sim \normal(0, \sigma^2 I)}[f(\vect{x} + \vect z)]
    \end{aligned}
\end{equation}
 Therefore, instead of the prediction at a data point $\vect x$, the expected prediction is computed for noisy samples generated by a normal distribution centered at $\vect x$. %
 The objective is to find $\mathcal L (\mathcal{B}_p(\epsilon)) := \min_{\delta \in \mathcal{B}_p(\epsilon)} g(\vect{x} + \bm{\delta})$ and evaluate if $\mathcal L (\mathcal{B}_p(\epsilon)) > \frac{1}{2}$, where $\smash{\mathcal B_p(\epsilon): = \{\bm{\delta} \in \mathcal X | \; \lvert\lvert \bm{\delta} \rvert\rvert_p \leq \epsilon\}}$. In other words, we verify if the prediction for all perturbations $\bm\delta$ is larger than 0.5, which corresponds to remaining the same. 
 To solve the optimization problem efficiently, one can evaluate the lower bound of $\smash{\mathcal L (\mathcal{B}_p(\epsilon)) \geq \mathcal L (\mathcal F, \mathcal{B}_p(\epsilon))}$ by finding the minimum value for the worst-case classifier as follows \cite{Zhang2020BlackBoxCW}:
\begin{equation}
    \label{beygauss:lower_app}
    \begin{aligned}
        \mathcal L (\mathcal F, \mathcal{B}_p(\epsilon)) := &\min_{\substack{\bm{\delta} \in \mathcal{B}_p(\epsilon) \\ h \in \mathcal F} } \mathbb E_{\vect z \sim \normal(0, \sigma^2 I)}[h(\vect x + \bm{\delta} + \vect z)] \\
        &\text{s.t.} \quad  \mathbb E_{\vect z \sim \normal(0, \sigma^2 I)} [h(\vect x + \bm{\delta})] = g(\vect x)
    \end{aligned}
\end{equation}
Here, $\mathcal F := \{f| f: \mathcal X \mapsto [0,1]\}$ is the set of classifiers. $\Phi$ denotes the CDF of the standard normal distribution $\normal(0,1)$, and $p_{\vect x}^*$ is the lower bound of $g(\vect x)$. We can certify that $g$ is certifiably robust if $\epsilon < \sigma\Phi^{-1}(p_{\vect x}^*)$ \footnote{Here it is for a $1D$ classifier, i.e, $\mathcal X = \mathbb R$} \cite{cohen2019, Zhang2020BlackBoxCW}. We show the resulting worst-case classifier, $h^* \in \mathcal F$ leading to this guarantee in \cref{fig:worstcase_app}. It is a linear classifier with all the points predicting class $1$ clustered behind a linear decision boundary such that $\mathbb E_{\vect z \sim \normal(0, \sigma^2 I)}[h^*(\vect x + \vect z)] = 
 g(\vect x)$. The worst-case smooth classifier will predict class $1$ until its mean $h^*(\vect x + \epsilon)$ is shifted to this boundary at $\vect x + \sigma \Phi^{-1}(g(\vect x))$, at which point both classes will be predicted with equal probability. Hence, any lower bound $p_{\vect x}^*$ of $g(\vect x)$, $\epsilon < \sigma\Phi^{-1}(p_{\vect x}^*)$ guarantees robustness. 

\begin{figure}[t]
\vskip 0.2in
\begin{center}
\centerline{\includegraphics[width=.9\columnwidth]{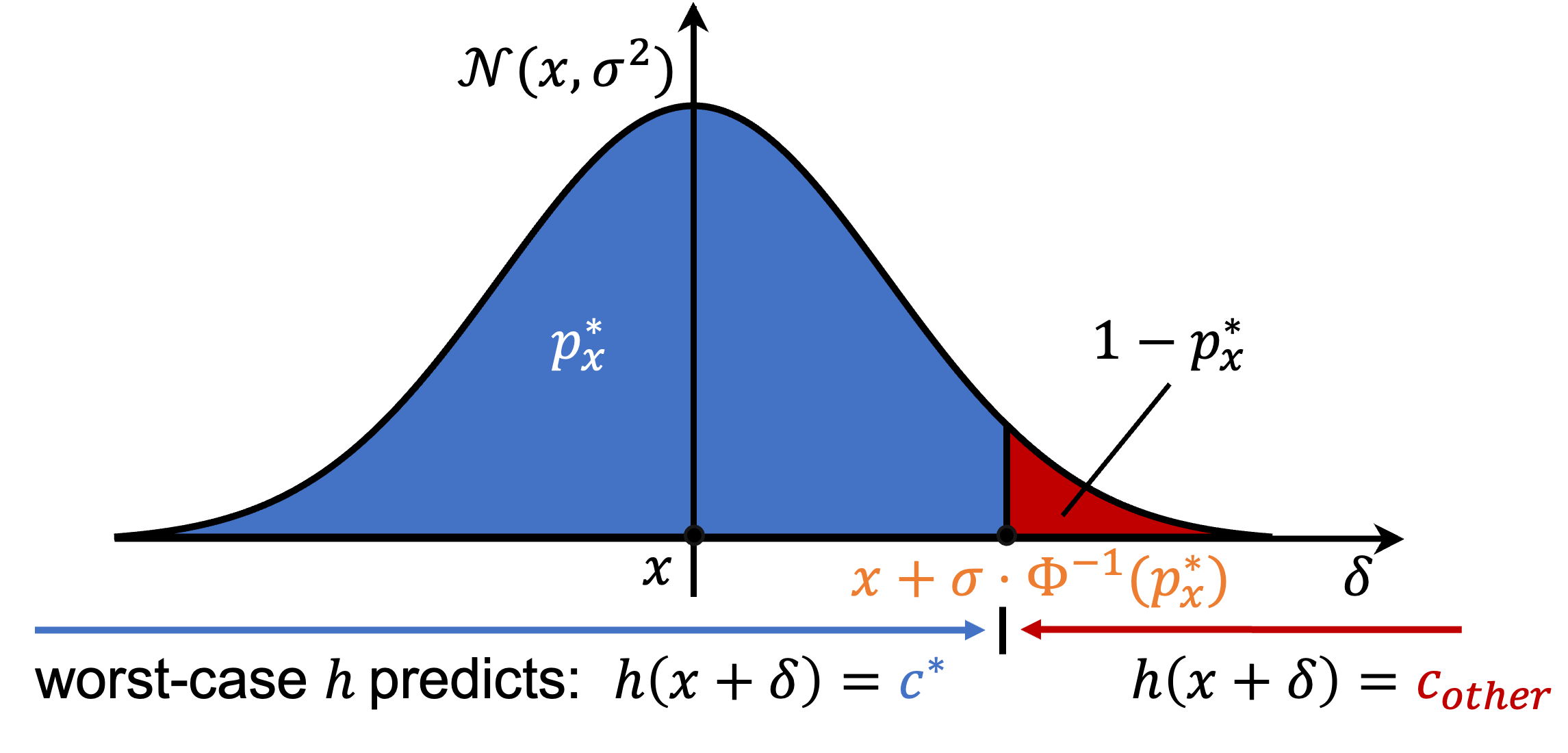}}
\caption{Example of the worst-case classifier in one dimension; $c^* = 1$ and $c_{\text{other}} = 0$. In this figure $h(.)$ implies $h^*(.)$.}
\label{fig:worstcase_app}
\end{center}
\vskip -0.2in
\end{figure}

\subsection{Discrete Randomized Smoothing}
\label{subsec:drs}
In many fields, we encounter discrete data for which continuous noise is not applicable. 
To address this challenge, \citet{bojchevski2023efficient} introduced an efficient randomized smoothing-based certification algorithm tailored for all forms of discrete data. They introduced a sparsity-aware smoothing scheme as a means of constructing a smooth classifier, which can then be efficiently certified. For simplicity, we describe the method for binary data, although the concepts can be extended to general discrete data.

For given numbers of additions and deletions, $r_a$, $r_d$, one defines the set of semantically similar graphs as those obtained by adding and deleting less than $r_a$ and $r_d$ edges respectively.  For any binary data, this set is formally defined as:
\begin{equation}
    \label{eq:pertball_disc}
\begin{aligned}
    \mathcal{B}_{r_a, r_d}(\vect x):= \{\vect {\tilde x} \in \mathcal X_D \text{ s.t. } & \sum_{\vect{\tilde x}_i  = \vect x_i - 1} 1 \le r_d 
    & \sum_{\vect{\tilde x}_i  = \vect x_i + 1} 1 \le r_a 
    \}
\end{aligned}
\end{equation}
Notably, the only possible perturbation on the binary data is to flip the bits, and consequentially for smoothing one can model the perturbation on each component as a data-dependent Bernoulli distribution. More precisely, each dimension of the discrete data $\tilde{\vect x}_i$ given the input data $\vect x$ is perturbed as in \cref{eq:drs_dist}, where each bit is flipped with probability $p_-$ if the actual value is $1$ and $p_+$ otherwise. Formally, this writes:
\begin{equation}
\label{eq:drs_dist}
\tilde{\vect x}_i| \vect x \sim \text{Ber}(p_{+}\cdot\mathbb I (\vect x_i=0) + (1 - p_{-})\cdot\mathbb I (\vect x_i=1)).
\end{equation}
We obtain the joint Probability Mass Function (PMF)
\footnote{often while writing we skip $p_+$ and $p_-$} as:
\begin{equation}
    \label{eq:sparsity_aware}
    \phi(\vect{\tilde x}|x; p_+, p_-) = \prod_{j=0}^{n-1} (P_{F}^{(j)})^{\mathbb I(\vect{\tilde x}_{j} \ne \vect {x}_{j})} (1-P_{F}^{(j)})^{\mathbb I(\vect{\tilde x}_{j} = \vect x_{j})}.
\end{equation}
Here $P_F^{(j)} := p_{+}\cdot \mathbb I (\vect x_j=0) + p_{-}\cdot \mathbb I (\vect x_j=1)$ is the probability of a bit-flip. Contrary to Gaussian smoothing in the continuous Euclidean domain, we define $g:\mathcal X \mapsto [0,1]$ as:
\begin{equation}
    \label{eq:drs_smoothclass}
    \begin{aligned}
                g(\vect x) &=  \Prob [f(\vect{\tilde{x}} ) = 1] \quad \text{for}\quad \vect{\tilde x} \sim \phi(.|\vect x)\\
        &= \sum_{f(\vect{\tilde x}) = 1} \phi(\vect{\tilde x}|\vect x).
    \end{aligned}
\end{equation}
Although the evaluation of the discrete smooth classifier for binary data is computable, it becomes intractable with increasing dimension of the data space. Therefore, stochastic methods such as the MC method with Clopper-Pearson intervals are employed in practice. \citet{bojchevski2023efficient} propose an efficient algorithm to certify the smooth classifier whose run time varies linearly with the radius of certification and is independent of the size of the data. %

\label{main_text_quantum_background}

\subsection{Quantum Phase Estimation}
\label{phaseestimation}

Given a unitary operator $U$, characterized by eigenvalues in the form $\exp(2 \pi i \theta)$, phase estimation emerges as an important quantum algorithm to efficiently find the phase $\theta$ associated with these eigenvalues. This algorithm serves as a fundamental subroutine in various quantum methodologies, including Shor's factoring algorithm \cite{Shor_1997} and Quantum Amplitude Estimation \cite{Brassard2000QuantumAA}.

The core idea of phase estimation involves the iterative application of controlled operations
, typically denoted as $CU$, to the corresponding eigenstate, progressively doubling the number of applications of the controlled gate with each iteration. If the phase can be accurately represented using $t$ bits, after $t$ iterations of applying the $CU$ gate, we obtain a Quantum Fourier Transform (QFT) of the phase. Subsequently, the Quantum Inverse Fourier Transform (QIFT) is used after applying controlled gates, followed by a measurement operation to recover the phase $\theta$. The circuit for Phase Estimation is shown in \cref{fig:pecqis}.

\section{Discrete Quantum Randomized Smoothing}
\label{quantum_certification}

In this section, we go into the details of our methods and our main contributions. We present our quantum algorithm for evaluating the discrete smooth classifier that is provably faster, providing a quadratic speed-up over the classical MC-based algorithm. %
We provide an overview of our approach in \cref{alg:QDRS}, further laying out the details in this section. From a high-level viewpoint, our algorithm works as follows:
\begin{enumerate*}[label=(\roman*)]
    \item we encode all perturbations as quantum superposition $\ket{\Psi(\vect x)}$ such that the probability of measuring the perturbed state $\ket{\Psi(\vect{\tilde x_i})}$ is the same as the probability of transition $\phi(\vect{\tilde x_i}| \vect x)$.

    \item We interpret the smooth classifier as the projection of the superposition $\ket{\Psi(\vect x)}$ onto the subspace spanned by all the perturbations for which the base classifier $f$ predicts $1$.

    \item Based on this interpretation, we construct an operator whose eigenvalues encode the value of the smooth classifier.

    \item Lastly, we evaluate the eigenvalue of this operator using the QAE.  
\end{enumerate*}

\begin{algorithm}[bt]
   \caption{Smooth Quantum Classifier (SQC)}
   \label{alg:QDRS}
\begin{algorithmic}
   \STATE {\bfseries Input:} binary data $\vect x$, classifier as oracle $\mathcal O_f$, flip probabilities $p_+, p_-$, counting qubits $n_q$, confidence $1-\delta$.
   \vspace {1 mm}
   \STATE $U = \textsc{GetDistributionLoaderCircuit()} $.
   \STATE $G = \textsc{GetGroverOperator}(U, \mathcal O_f)$.
   \STATE $\textsc{pec} = \textsc{GetPhaseEstimationCircuit}(G, n_q)$.
   \STATE $\textsc{BindParameters}(\textsc{pec}, \vect x, p_+, p_-)$.
   \vspace {1 mm}
   \FOR{$i=0$ {\bfseries to} $\left \lceil 17\log(\frac{1}{\delta}) \right\rceil $}

      \item  $\textsc{phase}[i] = \textsc{RunAndMeasure}(\textsc{pec})$     
   \ENDFOR
   \vspace {1 mm}
   \STATE $\textsc{sqc} = \textsc{Postprocess}(\textsc{phase})$
   \vspace {1 mm}
   \STATE {\bfseries Return: } $\textsc{sqc}$
\end{algorithmic}
\end{algorithm}

\subsection{Discrete smoothing and QAE}
The evaluation of the smooth classifier in \cref{eq:drs_smoothclass} involves looking at all $2^n$ perturbed states of the input data point $\vect{x}$ and adding the transition probabilities $\phi(\vect{\tilde x} | \vect{x})$ for the states $\vect{\tilde{x}}$ that belong to class $1$. The dimension of the Hilbert space representing quantum states with $n$ qubits is equal to the number of perturbations in binary data of length $n$. This allows us to encode all the perturbed states as orthogonal quantum states\footnote{Any two distinct strings have orthogonal basis encoding, i.e., for $\vect a $, $\vect b$ $\in \{0,1\}^n$ we have $\braket{\vect a}{\vect b} = \mathbf{1}[\vect a = \vect b]$.} using $n$ qubits with a basis embedding.  

Let $\vect{ \tilde x_i}$ be the perturbed data, with $i \in [0,2^{n-1}]$ being the decimal representation of $\vect{ \tilde x_i}$. The state $\ket{\vect {\tilde x}}$ is defined as $\ket{i_1}\otimes\ket{i_2}\otimes...\otimes\ket{i_{n}}$, where $i_1, i_2, ..., i_{n}$ constitutes the perturbed binary string. We propose a parameterized circuit formally as in Lemma \ref{lemma:dist_load} to prepare a superposition state which encodes all perturbations $\ket{\vect {\tilde x_i}}$ of the state $\ket{\vect x}$ as follows: 
\begin{equation}
\label{eq:desired_superposition}
\ket{\Psi(\vect x)}:= \sum_{\vect{i}=0}^{2^n-1} \sqrt{\phi( \vect {\tilde x_i}|\vect x)}\ket{\vect {\tilde x_i}}    
\end{equation}
We define the superposition of all \textit{good} (where $f$ predicts $1$) states and all \textit{bad} (where $f$ predicts $0$) states respectively  as:
\begin{equation}
\label{eq:GoodBadSpplit}
    \begin{split}
        \ket{\vect a} & :=  \frac{\sum_{f(\vect {\tilde x_i})=1}\sqrt{\phi(\vect {\tilde x_i}|\vect x)}\ket{\vect {\tilde x_i}}}{\sqrt{\sum_{f(\vect {\tilde x_i})=1}\phi(\vect {\tilde x_i}|\vect x)}} \\
        \ket{\vect b} & :=  \frac{\sum_{f(\vect {\tilde x_i})\ne1}\sqrt{\phi(\vect {\tilde x_i}|\vect x)}\ket{\vect {\tilde x_i}}}{\sqrt{\sum_{f(\vect {\tilde x_i})\ne1}\phi(\vect {\tilde x_i}|\vect x)}}.
    \end{split}
\end{equation}

Note that $\ket{\vect a}$ and $\ket{\vect b}$ are orthonormal; therefore, using simple algebra, one can decompose $\ket{\psi(\vect x)}$ into orthogonal subspaces spanned by \textit{good} and \textit{bad} perturbation embeddings (\cref{eq:phaseestimation}). Essentially, given an input data point $\vect x \in \mathcal X_D$ and distribution parameters $p_-$ and $p_+$, one can use our \cref{lemma:dist_load} to efficiently construct a circuit that acts in the ground state $\ket{0}$ as follows:
\begin{equation}
\label{eq:phaseestimation}
    U(\vect x, p_-, p_+)\ket{0} = \sqrt{g(\vect x)} \ket{a} + \sqrt{1 - g(\vect x)} \ket{b}.
\end{equation}

Here $g(\vect x) := \sum_{f(\vect{\tilde x_i})=1}\phi(\vect {\tilde x_i}|\vect x)$ is the evaluation of the smooth classifier; see \cref{eq:drs_smoothclass}.
The task of computing the smooth classifier is reduced to finding the projection of the weighted superposition $\ket{\psi(\vect x)}$ onto the subspace of \textit{good} perturbations, i.e., $g(\vect x)$ from \cref{eq:drs_smoothclass} is the amplitude of $\ket{\vect a}$ in \cref{eq:phaseestimation}. This interpretation allows us to use QAE \cite{Brassard2000QuantumAA} to evaluate the smooth classifier. We now state the lemma describing a circuit to load all the perturbations as a superposition with the square of the amplitude being the transition probability $\phi(\vect {\tilde x_i}|\vect x)$, while we defer the proof to Appendix \ref{app:parametrized_loader_circuit}.

\begin{lemma}
\label{lemma:dist_load}
Let $\vect x \in \{0,1\}^n$ represent our binary data, $f: \vect x\mapsto \{0,1\}$ be the base classifier, and $p_{-}, p_{+} \in [0,1]$ be the probability of addition and deletion, respectively, as used in \cref{eq:drs_dist}. Define,  
\begin{equation}\begin{aligned}
    &\theta_0(p_+):= 2\arccos(\sqrt{1 - p_+})\\
    &\begin{aligned}
    \theta_D :&= -2(\arccos(\sqrt{1 - p_+}) + \arccos(\sqrt{1 - p_-}))\\ &= -2 \arccos(\sqrt{1 - p_+}\sqrt{1 - p_-} - \sqrt{p_+}\sqrt{p_-}),      
    \end{aligned}
\end{aligned}\end{equation}
Then the following holds\footnote{This is only true up to a global phase but as we discuss in the proof, it will still lead to the desired Grover's operator and therefore for all further discussions we ignore the global phase.} for the operator $$U(\vect x, p_-, p_+): = \otimes_{m=0}^{n-1}  RY(\theta_0 + \vect x_m \theta_D) \circ RX(\vect x_m \pi)$$ and $\ket{\Psi(x)}$ as defined in \cref{eq:desired_superposition}. 
\begin{equation}
    \label{eq:udef}
    U(\vect x, p_-, p_+)\ket{0} = \ket{\Psi(x)}
\end{equation}
\end{lemma}

Furthermore, we can construct an operator whose eigenvalues encode the value of the smooth classifier and employ the QAE algorithm to solve the problem. However, there are alternative algorithms that can replace QAE~\cite{approximatecounting, iterrativecounting}. 

\subsection{Smooth classifier as the phase of a quantum operator}
\label{DRS_AS_GROVER}
In this subsection, we construct a quantum operation whose eigenvalues encode the value of the smooth classifier. Let $\mathcal R _0:= 2\ket{0}\bra{0} - \mathbb I $ and $\mathcal R _{\psi_x}:= 2\ket{\psi(\vect x)}\bra{\psi(\vect x)} - \mathbb I $ represent the reflections about $\ket{0}$ and $\ket{\psi(\vect x)}$, respectively. Using \cref{eq:udef}, we can readily deduce that:
\begin{equation}
    \label{eq:ref}
    \mathcal R _{\Psi_{\vect x}} = U(\vect x, p_-, p_+)\mathcal R _{0}U(\vect x, p_-, p_+)^H
\end{equation}
Given an oracle $O_f$ that can simulate the action of the classical function $f$, i.e., it can realize the following operation 
$O_f:\ket{\vect x} \mapsto (-1)^{f(\vect x)}\ket{\vect x}$. Its actions on $\ket{\vect a}$ and $\ket{\vect b}$ can be described as $O_f\ket{\vect a} = -\ket{\vect a}$ and $O_f\ket{\vect b} = \ket{\vect b}$. As a result, for any quantum state $\ket{\Phi}$ within the span of $\ket{\vect a}$ and $\ket{\vect b}$, the application of $O_f$ reflects $\ket{\Phi}$ with respect to $\ket{\vect b}$. Defining Grover's operator $G := \mathcal R_{\Psi_{\vect x}} O_f$, it follows that $G$ rotates any state $\phi \in \text{span}(\ket{\vect a}, \ket{\vect b})$ by an angle of $2\theta$ toward $\ket{\vect a}$, where $\theta$ denotes the angle between $\ket{\Psi(\vect x)}$ and $\ket{\vect b}$. Consequently, within the subspace spanned by $\ket{\vect a}$ and $\ket{\vect b}$, the operator $G$ works as a rotation operation with an angle of $2\theta$, resulting in eigenvalues of $\exp(\pm \iota 2 \theta)$ in this plane. Using the representation of $\ket{\Psi(\vect x)}$ as specified in \cref{eq:phaseestimation}, along with the orthonormality of $\ket{\vect a}$ and $\ket{\vect b}$, we can express the output of the smooth classifier as follows: 
\begin{equation}    \label{eq:quantum_smooth}
    g(\vect x) = \sum_{f(\vect{\tilde x_i})=1} \phi(\vect{\tilde x_i}|{\vect x}) = |\braket{\vect a|\Psi(\vect x)}|^2 = \sin^2(\theta)
\end{equation}
Therefore, finding the output of the smooth classifier $g(\vect x)$ reduces to finding the phase of Grover's operator $G$ corresponding to the eigenvalues in $\texttt{span}(\ket{a}, \ket{b})$. We employ the phase estimation algorithm briefly described in \cref{phaseestimation} to find the phase of the desired Grover operator. Using the results of \citet{Brassard2000QuantumAA}, one can formally estimate the order of convergence of the quantum algorithm $SQC$ from \cref{alg:QDRS} as \cref{lemma:convergence_quantum}. Finally, we state our main result on \textit{quadratic} speedup as \cref{thm:main}.

\begin{lemma}
\label{lemma:convergence_quantum}
    Given SQC($O_f, p_-, p_+$, $n_q$, $1-\delta$) as defined in \cref{alg:QDRS}. To approximate the lower (upper) bound $g_a(x)$ of $g(x)$, such that $|g_a(x) - g(x)| < \epsilon$ with a certainty of $1 - \delta$, it requires $O(\ln(\frac{1}{\delta})\frac{1}{\epsilon})$ calls to the oracle $O_f$ simulating the base classifier $f$.

\end{lemma}
\begin{theorem}
\label{thm:main}[Quadratic Speedup]
SQC($O_f, p_-, p_+$, $nqc$, $1 - \delta$) as mentioned in Algorithm \ref{alg:QDRS} needs quadratically fewer calls to the base classifier compared to the classical Monte-Carlo method with Clopper-Pearson confidence intervals to obtain an estimate of $g(x)$ within the same accuracy.
\end{theorem}
We outline the proofs in \cref{proofs} while building an intuitive understanding of the convergence result in \cref{phase_estimation}. We elaborate on deriving the guaranteed lower bound of the exact smooth classifier in \cref{phase_estimation}.

\section{Discrete Perturbations on Continuous Data}
\label{discrete_data_main_certification}

In the following, we adapt our framework to a new threat model. The data are continuous, but the perturbations of the adversary are discrete.
For example, we have an RGB image comprising $M \times M$ pixels, where the adversary can corrupt up to $N$ pixels. While the data space expands to $\mathbb R^{3 \times M \times M}$, the set of perturbations to a fixed image consists of $2^N$ elements--each pixel can be corrupted or remains unaffected). Each perturbation of an image can be uniquely represented by a binary string of length $N$. For instance, setting the $i^{th}$ bit to $1$ indicates the corruption of the $i^{th}$ pixel in the given image. We demonstrate the application of discrete randomized smoothing in this context to certify robustness against potential binary perturbations.

\subsection{Threat Model and Smoothing Distribution}
In an $N$ dimensional data space, the perturbation model $p(\cdot,\cdot)$ allows the adversary to affect the $i^{th}$ feature $x_i$ of the input data $\vect x$ as $p(x_i, i)$. We define the set of perturbations as:
\begin{equation}
   \mathcal P_{\vect x} := \{\vect {\tilde x} \in \mathcal X | \tilde x_j \in \{x_j, p(x_j, j)\}\}. 
\end{equation}
Based on the set of perturbations $\mathcal P_{\vect x}$ we define the threat model as the ball of radius $r$ around $\vect x$ in \cref{eq:pertball_disc_pert}, and the corresponding smoothing distribution as \cref{eq:discrete_perturbation_dist}.  
\begin{equation}
    \label{eq:pertball_disc_pert}
\begin{aligned}
    \mathcal{B}_{r}(\vect x):= \{\vect {\tilde x} \in \mathcal P_{\vect x} \text{ s.t. } & \sum_{\vect{\tilde x}_i  = p(\vect x_i,i)} 1 \le r \}
\end{aligned}
\end{equation}
\begin{equation}
\label{eq:discrete_perturbation_dist}
\hat \phi(\vect {\tilde x} | \vect x; p)  := (p)^{\sum_{\vect{\tilde x}_i  = p(\vect x_i,i)} \mathbf{1}}(1 - p)^{\sum_{\vect{\tilde x}_i  = \vect x_i} \mathbf{1}}
\end{equation}

\subsection{Smooth Classifier and Robustness Certification}
Let $\hat \classifier : \mathcal X \mapsto \{0,1\}$ denote the base classifier, we propose to define the smooth classifier as 
\begin{equation}
\label{eq:drs_class_discrete}
\hat g(\vect x) := \sum_{\vect{\tilde x} \in \mathcal P_{\vect x} \text{ and } \hat \classifier (\vect{\tilde x}) = 1} \hat \phi(\vect {\tilde x} | \vect x; p)    
\end{equation}
Further, to certify $\hat g(\vect x)$  we establish a bijective mapping $\eta_x$ 
 from $\mathcal{P}_{\vect x}$ onto $\{0,1\}^N$. This mapping ensures that the distribution $\hat \phi(\vect{\tilde x} | \vect x; p)$ coincides with $\phi(\eta_{\vect x}(\vect{\tilde x}) | \eta_{\vect x}(\vect x); p_+ = p, p_-=0)$ which is defined in \cref{eq:sparsity_aware}. Additionally, for a given $\vect x$, we define a classifier $\classifier$ on $\{0,1\}^N$ based on the base classifier $\hat \classifier$ on $\mathcal X$, as $\classifier = \hat \classifier \circ \eta_{\vect x}^{-1}$. By certifying the smooth classifier, we establish analogous guarantees for $\hat g$. The formal description of this approach is provided in \cref{discrete_perturbations}. Furthermore, as discussed in \cref{discrete_perturbations}, there's no necessity to construct the superposition of all perturbed states. Instead, we generate the superposition of perturbation representations $\eta_{\vect x}(\vect{\tilde x})$ utilizing \cref{lemma:dist_load}, and integrate it with $\vect x$ during oracle evaluation.

\section{Experiments}
\label{experiments}
In this section, we demonstrate the utility of our proposed framework and examine its certification strength in multiple domains, including binary images, graphs, and natural language processing. %
Within the realm of discrete perturbations framework (\cref{discrete_data_main_certification} and \cref{discrete_perturbations}), we choose two settings: \begin{enumerate*}[label=(\roman*)]
    \item Image classification on the binary MNIST dataset, 
    \item Sentiment analysis on the IMDB reviews.
\end{enumerate*} We also evaluate our framework on the discrete data setting: \begin{enumerate*}[resume, label=(\roman*)] 
    \item Graph classification task to find a $ 4-$ clique within the collection graphs with $6$ nodes.
    We present additional results with more details on individual experiments in \cref{extra_experiments}. 
\end{enumerate*}

We compare the performance of the quantum smooth classifier against the exact smooth classifier, which can be evaluated because of the limited perturbation representation. Moreover, we do not implement the classical models as quantum operations directly; instead, we construct a diagonal unitary that mimics the model giving the same output at all the perturbed data points. All experiments are carried out as a QC simulation on Qiskit with Aer backend~\cite{Qiskit} on an A100 node with up to 1TB of memory. We use SEML to manage our slurm experiments~\cite{seml_2022}%

\subsection{Binary-MNIST Classification}
We binarise the images of the MNIST handwritten digit dataset \cite{deng2012mnist} and rescale them to $16 \times 16$. The certification of the smooth classifier requires computing the output corresponding to each class. Therefore, to simplify the problem computationally, we stick to the binary classification of predicting $4$ or $\neg 4$.
To construct the smooth classifier, we use $p_{+} = p_{-} = 0.3$ as flip probabilities. We define a perturbation window with $17$ pixels, as shown in \cref{fig:binary_MNIST_window}, and we certify against perturbations within that window. We then randomly select $50$ images from the validation split to evaluate our algorithm. Figure \ref{fig:binary_MNIST} shows one of the validation images, the perturbation window, and a random perturbation to the given validation image within the perturbation window. 

\begin{figure}[t]  %
    \centering
   \subfigure[]{\includegraphics[width=0.32\linewidth]{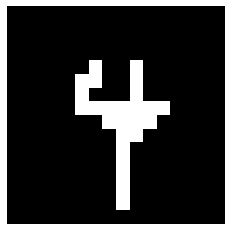}\label{fig:binary_MNIST_4}}
    \subfigure[]{\includegraphics[width=0.32\linewidth]{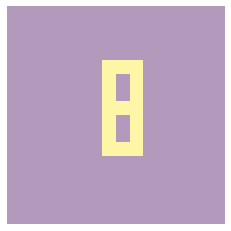}\label{fig:binary_MNIST_window}}
    \subfigure[]{\includegraphics[width=0.32\linewidth]{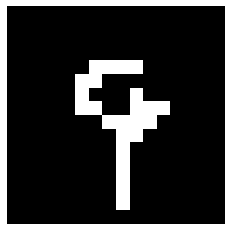}\label{fig:binary_MNIST_4_pert}}

    \caption{(a) Randomly selected validation image; (b) Window of perturbation, only the highlighted pixels are allowed to be perturbed and the subsequent robustness guarantees are only against attacks within that window; (c) Validation image randomly perturbed within the window.}
    \label{fig:binary_MNIST}        
    
\end{figure}

In \cref{fig:heatmap_binary_MNIST} we show heatmaps depicting the percentage of certified images for varying perturbations and with $5$, $6$, and $7$ counting qubits and the actual smooth classifier. The number of counting qubits represents the binary string corresponding to the phase. When we increase the amount, we see a significant improvement in the certification performance. In \cref{fig:MNIST_comparison}, we evaluate the proportion of certified images relative to the overall perturbations ($r_a + r_d$), considering various numbers of counting qubits alongside the exact smooth classifier. We observe that we already achieve a decent approximation using only 7 counting qubits.
\begin{figure}[t]  %
    \centering
    \subfigure[]{\includegraphics[width=0.4\linewidth]{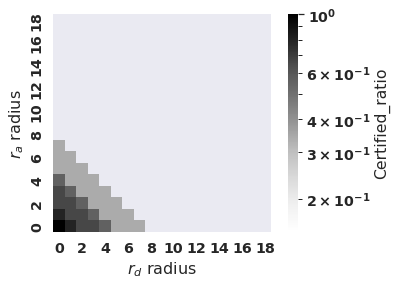}\label{fig:heatmap_binary_MNIST_5}}
    \subfigure[]{\includegraphics[width=0.4\linewidth]{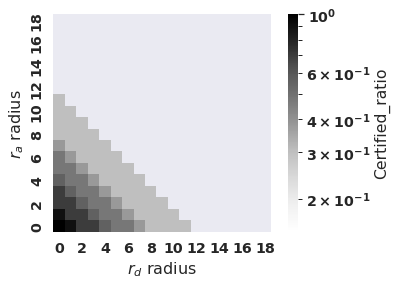}\label{fig:heatmap_binary_MNIST_6}}
    \subfigure[]{\includegraphics[width=0.4\linewidth]{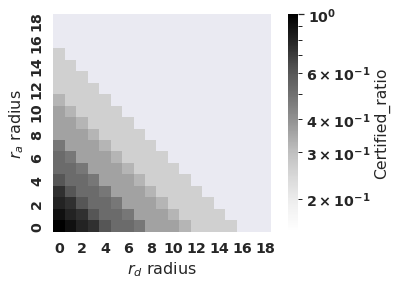}\label{fig:heatmap_binary_MNIST_7}}
    \subfigure[]{\includegraphics[width=0.4\linewidth]{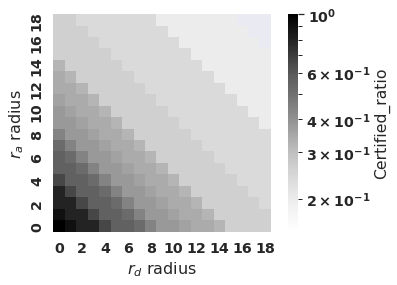}\label{fig:heatmap_binary_MNIST_exact}}
     \caption{Heatmaps showing the percentage of 50 validation images in consideration that are robust against the indicated set of perturbations in the grid corresponding to the quantum smooth classifier($p_- = 0.3, p_+=0.3$) with (a) $5$ (b) $6$ (c) $7$ counting qubits and (d) corresponding to the actual smooth classifier.}
    \label{fig:heatmap_binary_MNIST}        
    \vspace{-0.2in}
\end{figure}
\begin{figure}[tbh]
\vskip 0.2in
\begin{center}
\centerline{\includegraphics[width=0.75\columnwidth]{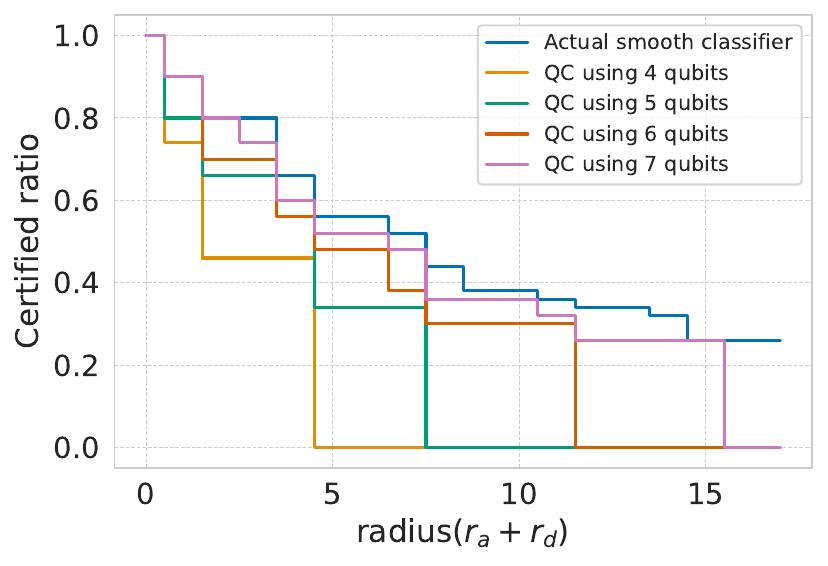}}
\caption{Certified ratio, exact classifier vs quantum classifiers ($p_- = 0.3, p_+=0.3$) using $4$, $5$, $6$ and $7$ counting qubits evaluated for randomly selected $50$ images.}
\label{fig:MNIST_comparison}
\end{center}
\vskip -0.2in
\end{figure}
In \cref{fig:Convergence_MNIST}, we present the empirical comparison between the classical and quantum classifier for the convergence rate of the \textit{average error} for \textit{the number of calls to the oracle}. We observe that, although the error is still higher for the quantum approach, the slope of the quantum model is strongly favorable. For increased dimensionality of the data, our curve stays the same while the classic randomized smoothing is likely to get worse. From the slopes in \cref{fig:Convergence_MNIST}, we can calculate that if the quantum algorithm needs $N$ samples, then the classical algorithm takes $O(N^{1.92})$ samples.

\begin{figure}[t]
\vskip 0.2in
\begin{center}
\centerline{\includegraphics[width=0.75\columnwidth]{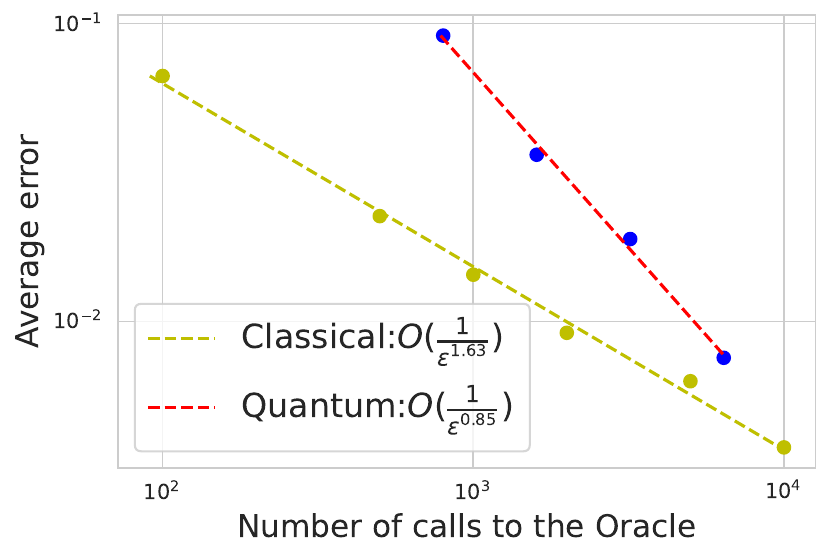}}

\caption{Convergence of error with number of calls to the oracle. Comparison for \textbf{Classical} vs \textbf{Quantum} algorithm with errors averaged over $50$ images from the MNIST testset.}

\label{fig:Convergence_MNIST}

\end{center}
\vskip -0.2in
\end{figure}

\subsection{Graph Classification}
We further investigate discrete quantum randomized smoothing for graph classification. A subfield of machine learning on graphs focuses on detecting subgraph structures within a graph with graph neural networks to determine their expressivity~\cite{chen2020graph, Bouritsas2020ImprovingGN, campi2023expressivity}. Building upon this concept, we develop a dataset specifically tailored for binary classification by generating graphs that either contain a 4-clique or lack one entirely. These graphs are constructed using the Erdős–Rényi (ER) model with six nodes. For graphs belonging to class 0, we use a connection probability of $65\%$, while those in class 1 are assigned a connection probability of $30\%$. We showcase an example of each class in Figure \ref{fig:graph_dataset}. We trained a PPGN~\cite{maron2020provably} graph neural network on this data as the base classifier. Then we apply the smoothing approach. 

\begin{figure}[tbh]  %
    \centering

   \subfigure[]{\centering\includegraphics[width=0.3\linewidth]{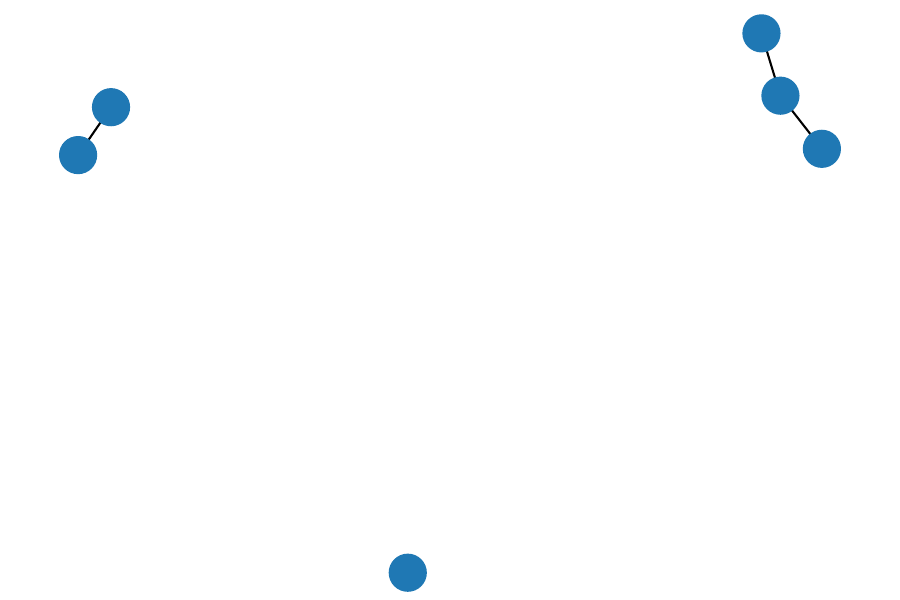}}\hspace{0.5in}
    \subfigure[]{\centering\includegraphics[width=0.3\linewidth]{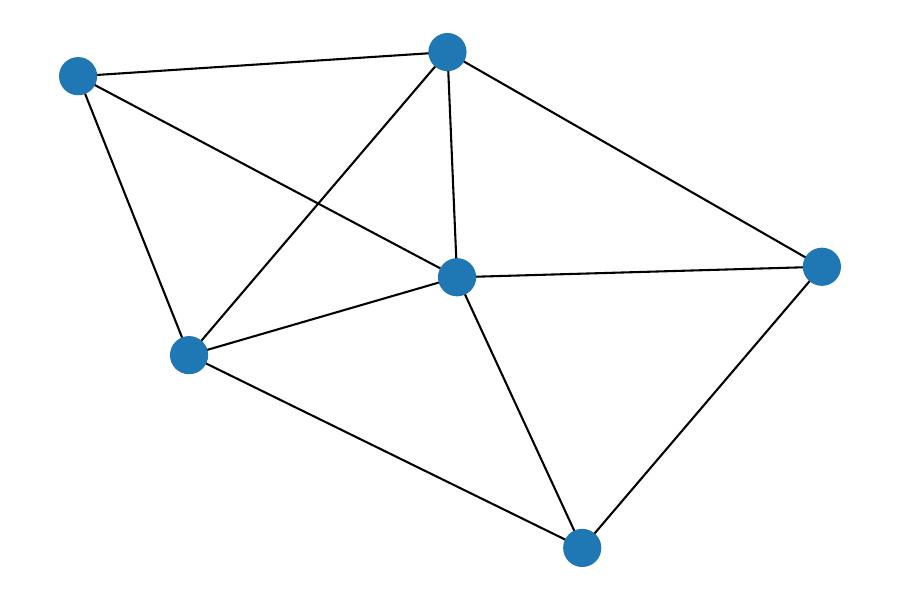}}

\caption{Example images of the graph classification dataset. On the left, we see a graph from class 1. These graphs do not contain a 4-clique. The right graph is an instance of class 2. All of these contain a 4 clique.}

\label{fig:graph_dataset}

\end{figure}

To smooth the data, we use $p_+=0.3$ and $p_- = 0.0$, since any deletion of an edge has a high probability of destroying a potential 4-clique. Thus, because of the absence of noise for edge deletions, we cannot certify against any adversarial deletion. The line plot in \cref{fig:graphs_actual_classifier_approximation} shows the behavior of our simulated quantum approach for $4-8$ counting qubits. We compare it with the actual smooth classifier that acts as a ground-truth result of the classic discrete smoothing approach~\cite{bojchevski2023efficient}. We can certify most of the graphs for two additions, while some graphs can even be certified for up to 6\footnote{Graphs of class 0 already containing a 4-clique can have as many edges added as possible.}. As we increase the number of counting qubits to $8$, our approximation is very close to the ground truth values, see \cref{fig:graphs_actual_classifier_approximation}.  We observe a significant difference between the quantum classifier using $8$ counting qubits and the exact smooth classifier arising at radius 7, where our approach is unable to certify. %
Given that our approach consistently predicts a lower bound, it will never predict 1 for cases where the value is precisely 1. Hence, our certification will hold only for finite perturbations. This limitation explains why we can only certify up to six additions, while the ground truth extends to infinity.%

\begin{figure}[t]
\vskip 0.2in
\begin{center}
\centerline{\includegraphics[width=0.75\columnwidth]{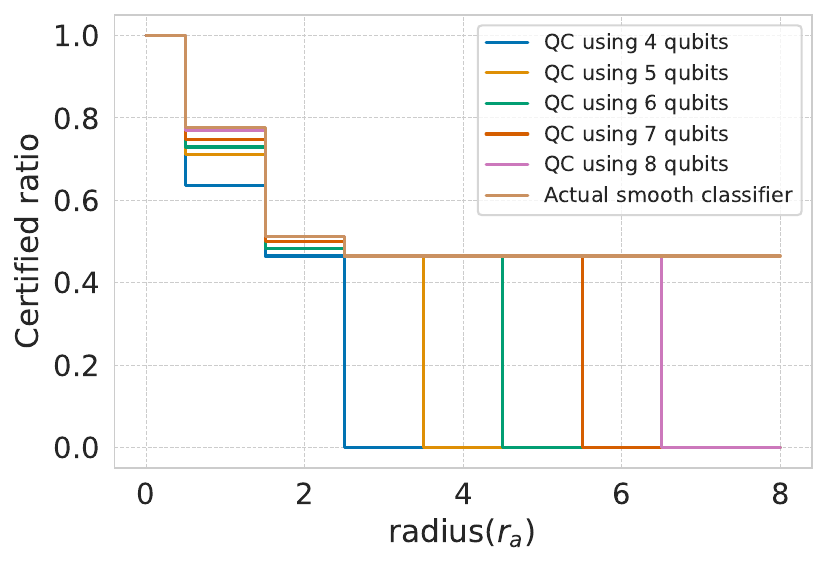}}
\caption{Certified ratio for the exact classifier and quantum classifier with $p_- = 0.0$, $p_+=0.3$ using four to eight counting qubits evaluated on $170$ randomly selected graphs.}
\label{fig:graphs_actual_classifier_approximation}
\end{center}
\vskip -0.2in
\end{figure}

\subsection{Sentiment Analysis}
Sentiment analysis is a subset of text classification in which aims to determine the polarity of sentences toward positive, negative, or neutral. We use a pre-trained model to predict the \textit{Positive} or \textit{Negative} tone of the IMDB reviews from Hugging Face and use it to construct a smooth classifier by randomly removing neutral words from the input sentences and finally compute certificates for the model's robustness against removal of such neutral words. The certification for a given input gives formal guarantees that the model will not change the prediction upon removal of a certain number of neutral words. We use a list of $29$ neutral words and randomly take $600$ ($300$ positively labeled and $300$ negatively labeled) reviews that have around $8-10$ neutral words from the list. 
We use $p_{rem}=0.2$ as the probability of removing neutral words. In \cref{fig:text_certificate}, we plot the percentage of certified sentences showing the strength of the certificate. It shows the increasing certification strengths for varying number of counting qubits. %

\begin{figure}[t]
\vskip 0.2in
\begin{center}
\centerline{\includegraphics[width=0.75\columnwidth]{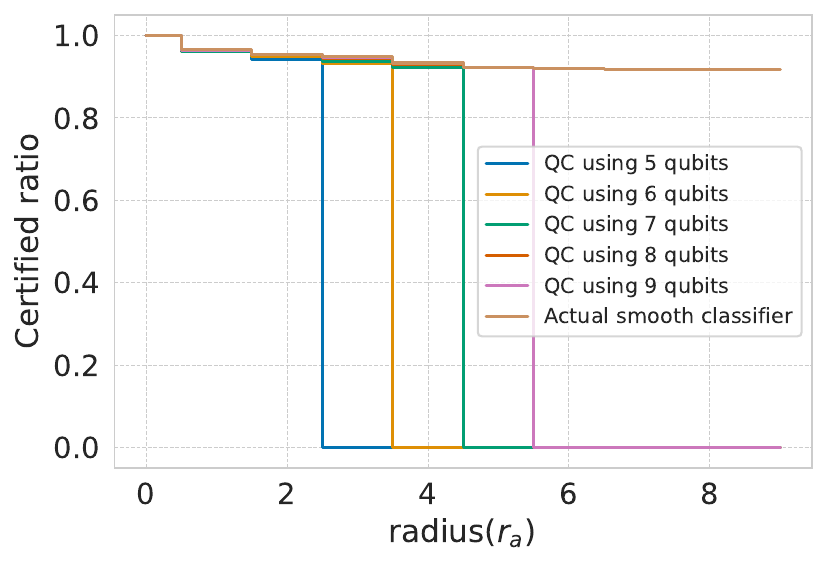}}
\caption{Certified ratio for the exact classifier vs quantum classifier with $p_{\text{rem}} = 0.2$ using five to nine counting qubits evaluated for $600$ randomly selected reviews with $8-10$ neutral words.}%
\label{fig:text_certificate}
\end{center}
\vskip -0.2in
\end{figure}

\section{Conclusion}
\label{conclusion}
In this work, we propose a novel framework that combines discrete randomized smoothing with quantum computing for robust classification of discrete data. We demonstrate a quadratic reduction in model evaluations using this framework. Additionally, we extend our approach to certify against threat models where data can be continuous but adversaries affect only restricted parts using discrete attacks. We evaluate the robustness of our approach in three different classification scenarios, showing that with five to seven counting qubits, a satisfactory approximation of the actual smooth classifier is achievable. Despite these improvements, our framework is limited as it handles either discrete data or continuous data with a restricted threat model as described in Section~\ref{discrete_data_main_certification}. Nonetheless, this study highlights the potential of discrete quantum randomized smoothing for certifying neural networks, opening avenues for further exploration and research.

\appendix

\subsection{Applying randomized smoothing to discrete
perturbations}
\label{discrete_perturbations}
We start by presenting a general framework that extends the idea of RS for discrete data to discrete perturbations and show its utility in verifying the correctness of some of the test cases presented in \cref{experiments}. We essentially construct a classifier on discrete data from the actual classifier, whose certification guarantees the certification of the original classifier for discrete perturbations. This allows us to validate our quantum approach on more realistic datasets. 
 
Assume that the input data lies in some $M-$dimensional data space $\mathcal{C}^M$, where $\mathcal C = \{0,1\}$, $\{0,1, ..., D\}$, $\mathbb Z$, $\mathbb Q$ , $\mathbb R$ or $\mathbb C$ and the adversary is allowed to affect only $N$ features\footnote{In the windowed MNIST experiment we allowed the adversary to perturb only a subset of the whole image. Therefore this framework can be used to verify its correctness.}, indexed as $x_{j_i}$ for $i = 0, ..., N-1$. Denote these set of indices as $\mathcal I_N := \{j_i\}_{i=0}^{N-1}$ and the remaining $M-N$ indices as $\overline{\mathcal I}_N := \{0,1,...,N-1\}/ \mathcal I_N$. Additionally, assume that the adversary can leave these features unaffected or apply a data-dependent perturbation $p(x_{j_i}, j_i)$. The perturbation model, denoted as $p(\cdot, \cdot)$ takes the $i^{th}$ feature $x_i$ to $p(x_i, i)$ and applying the perturbation model again takes it back to $x_i$. For example, \textit{non-corrupted} to \textit{corrupted}; $0$ to $1$ or $1$ to $0$; etc. Given $\vect x \in \mathcal C^M$, define the set of all perturbations on $\vect x$ as:
\begin{equation}
\label{eq:pert_set}
   \mathcal P_{\vect x} := \{\vect y \in \mathcal C^M |y_j = \left\{\begin{array}{lr}
       x_j & \text{for } j \in \mathcal I_N\\
       \in \{x_j, p(x_j, j)\} & \text{for } j \in \overline {\mathcal I_N}
        \end{array}\right\}. 
\end{equation}
Note that $\forall \vect z \in \mathcal P_{\vect x}$, $\mathcal P_{\vect z} = \mathcal P_{\vect x}$. This is because applying the perturbation model $p$ again will take $p(x_i, i)$ back to $x_i$. To construct a smooth classifier $\tilde g$, we can sample $\vect{\tilde{x}}$ from a distribution $\tilde \phi(.|\vect x)$ defined on $\mathcal P_{\vect x}$ and evaluate $\mathbb P(\tilde f(\vect{\tilde{x}} = 1))$ as our smooth classifier, where $\tilde f$ is the base classifier. Therefore as before the smooth classifier $\tilde g(\vect x)$ can be obtained as:
\begin{equation}
    \label{eq:discrete_pert_smooth}
    \tilde g(\vect x) := \sum_{\vect{\tilde x} \in \mathcal P_{\vect x} \text{ and } \tilde f(\vect{\tilde x}) = 1} 
 \tilde \phi(\vect{\tilde x| \vect x})
\end{equation}
We define $\tilde \phi$ to allow data-dependent perturbation using two distinct probabilities of perturbation $p_+$ and $p_-$. Let $q(,,.)$ determine the data-dependent probability, i.e. $q(x_i, i) \in \{p_+, p_-\}$ gives the probability of perturbation of the $i^{th}$ feature $x_i$, such that $q(p(x_i,i), i) = p_+$ if $q(x_i, i)  = p_-$ and $p_+$ otherwise. 

Given this setup, we want to argue the following, firstly \cref{eq:discrete_pert_smooth} can be certified using the framework of \cite{bojchevski2023efficient}
Secondly, it can be computed using our quantum algorithm without the need to construct a superposition of perturbed states $\vect{\tilde x} \in \mathcal P_{\vect x}$. We just need to construct the superposition in \cref{lemma:dist_load} and we can make data representation as part of the Oracle construction.

Clearly, $\lvert \mathcal P_{\vect x} \rvert = 2^N$, therefore we can find a bijection $\eta_x : \mathcal P_{\vect x} \mapsto \{0,1\}^N$. Define $\eta_x$ as follows:
\begin{equation}
\label{eq:pert_map}
    \eta_{\vect x}(\vect{\tilde x})_j = \left\{ \begin{array}{lr}
       1 & \text{if } q(x_j, j)  = p_-\\
       0 & \text{if } q(x_j, j)  = p_+
        \end{array}\right 
        \}
\end{equation}

It is easy to verify that \cref{eq:pert_map} is a bijection which allows us to define $f : \{0,1\}^N \mapsto \{0,1\}$ as a binary classifier on binary data such that $f(b) := \tilde f(\eta_{\vect x}^{-1}(b))$. The construction of $\eta_{\vect x}$ also ensures that $\forall \vect y, \vect z \in \mathcal P _{\vect x}$, $\tilde \phi(\vect y | \vect z) = \phi(\eta_{\vect x}(\vect y)|\eta_{\vect x}(\vect z))$, where $\phi$ is as defined in \cref{eq:sparsity_aware}. Using the framework of \citet{bojchevski2023efficient}, the smooth classifier defined as:
\begin{equation}
    \label{eq:disc_}
    g(\vect b) := \sum_{f(\vect{\tilde b}) = 1} 
\phi(\vect{\tilde b| \vect b})
\end{equation}
 can be certified. Since, \cref{eq:disc_} can be written as:
\begin{equation}
\begin{aligned}
    \label{eq:disc_0}
    g(\vect b) &= \smashoperator[r]{\sum_{\eta^{-1}_{\vect{x}}(\vect{\tilde b}) \in \mathcal P_{\vect x}, \tilde f(\eta^{-1}_{\vect{x}}\vect{\tilde b}) = 1}} 
    \tilde \phi(\vect{\eta^{-1}_{\vect x}(\vect {\tilde b})| \eta^{-1}_{\vect x}(\vect {b})}) = \tilde g(\eta^{-1}_{\vect x}(\vect {b})).     
\end{aligned}
\end{equation}
If changing $\vect b$ to $\vect{\tilde b}$ does not change the prediction of $g$ it can be inferred that changing $\eta_{\vect x}^{-1}{\vect b}$ to $\eta_{\vect x}^{-1}{\vect{\tilde b}}$ will not change the prediction of $\tilde g$. Therefore, any guarantees for $g$ will also hold correspondingly for $\tilde g$.
\paragraph{Implementing smooth classifier for binary perturbations on QC}
Assume we are given a classical neural machine learning model $\tilde \classifier$. For a given $\vect x$, we can implement a function $\eta_{\vect x}^{-1}$, and composing it with $\tilde \classifier$ we can implement $\classifier_{\vect x} := \tilde \classifier \circ \eta_{\vect x}^{-1}$. Employing the classical circuit that implements the base classifier $f_{\vect x}$, for the perturbation representation, we can construct the reversible circuit performing the following transformation: $(\vect b,y) \mapsto (\vect b,y\bigoplus f_{\vect x}(\vect b))$, as described in~\cite{nielsenchuang}. These reversible computations can then be simulated using quantum gates to obtain the corresponding quantum neural network oracle $O_{ f}$ capable of realizing the following map: $\ket{\vect b}\ket{-} \mapsto (-1)^{f_{\vect x}(\vect b)}\ket{\vect b}\ket{-}$. In principle, given any classical base classifier $\tilde f$, we can construct a quantum oracle $O_{f}:\ket{\vect x} \mapsto (-1)^{f(\vect x)}\ket{\vect x}$ with a comparable compute time. That is, we assume the compute time of the classical operation $\tilde f(\vect x)$ is similar to that of the corresponding quantum operation $O_f\ket{\vect b}$. This assumption allows us to expect that any asymptotic improvement in the number of calls to the oracle (or base classifier) will also be reflected in the actual runtime. Therefore we just need to prepare a superposition of all the perturbations as we can construct a data-dependent oracle that implicitly implements the base classifier on the data. For experiments, we have constructed a data-dependent diagonal unitary that given the perturbation representation gives the desired output.

\subsection{Phase estimation circuit and obtaining the lower bound}
\label{phase_estimation}
\begin{figure}
    \centering
    \includegraphics[width=\linewidth]{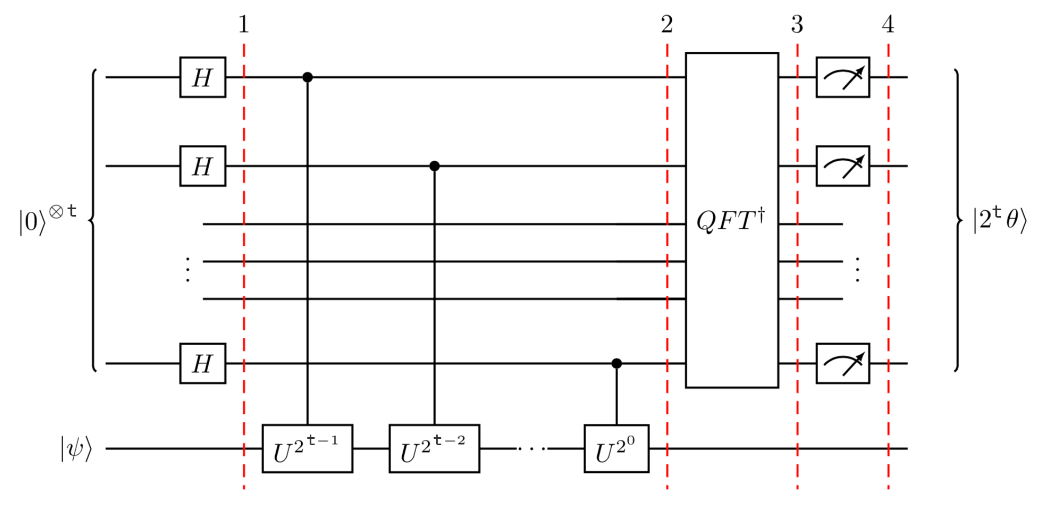}
    \vskip -0.15in
    \caption{A typical Phase Estimation Circuit, taken from~\cite{Qiskit}.}
    \label{fig:pecqis}
\end{figure}

By denoting the eigenvalues of the unitary operator $G$ as $\exp(\pm 2 \pi \iota \phi)$, we can apply the phase estimation algorithm to accurately estimate $\phi$ up to $t$ bits. Consequently, the value of the smooth classifier can be determined using the circuit depicted in Figure \ref{fig:pecqis}.

As discussed by \citet{bojchevski2023efficient}, obtaining a lower bound (or an upper bound when the output is below 0.5) on the probability that the smooth classifier predicts class $1$ is crucial for ensuring the certified robustness guarantees hold, irrespective of the base classifier selected. \citet{Brassard2000QuantumAA} previously analyzed to estimate the bounds for the absolute difference between the actual and predicted amplitudes, specifically, $|sin^2(\pi \phi_{actual}) - sin^2(\pi \phi_{approx})|$. In principle, the aforementioned error bound can be employed by subtracting it from the approximated value to acquire the lower bound. Nonetheless, utilizing this uniform error bound (independent of the actual value) might prove excessively conservative, particularly when working with a smaller number of qubits. As seen in our experiments(\cref{experiments}), such deviations from the actual value could lead the algorithm to disregard potential robustness. 

Essentially, the phase estimation algorithm, contingent upon the number of counting qubits, can yield a finite set of distinct output values ($\sin^2(\theta/2)$). These values partition the interval $[0,1]$ into a series of subintervals, denoted by $\{[a_j, a_{j+1}]\}_{j=1, ..., 2^{n-1}}$. Given that the second register in Figure \ref{fig:pecqis} is initialized with $\ket{\psi(\vect x)}$, representing an equal superposition of the eigenstates of the operator $G$, for every observed phase $\omega$, a similar number of observations is expected for $1 - \omega$ (except when both the eigenstates coincide). Nevertheless, it is crucial to acknowledge that both values produce identical outputs. As illustrated in~\cite{Brassard2000QuantumAA}, with a probability of at least $\frac{8}{\pi^2}$, boundary values of the interval $[a_{j^*}, a_{j^* + 1}]$ is obtained, such that $g(\vect x) \in [a_{j^*}, a_{j^* + 1}]$. Consequently, if a state corresponding to $a_i$ is observed, $a_{i-1}$ is a high-probability guaranteed lower bound. This probability is further enhanced by repeating the experiment multiple times. 

\textbf{Intutive understanding of \cref{lemma:convergence_quantum}:}
From our discussion above we notice that the error in estimating the precise value may diminish exponentially as the number of qubits increases, due to the exponential growth in the number of intervals. Nonetheless, the number of calls to the oracle (through Grover's operator G) also experiences exponential growth. As a result, a linear decrease in the error can be anticipated with the number of calls to the oracle.

\subsection{Proofs}
\label{proofs}
\subsubsection{Parameterized distribution loader circuit}\label{app:parametrized_loader_circuit}
\begin{proof}[Proof of Lemma \ref{lemma:dist_load}]
Here we that $U(x, p_-, p_+)$ creates the desired superposition state upto a global phase, and the global phase introduced in the process does not change the resulting Grover's operator.

Starting with the first part, decompose $U(x, p_-, p_+)$ as two operators $DL(x; p_-, p_+): = \otimes_{m=0}^{n-1} RY(\theta_0 + x_m \theta_D)$ and $D(x):= \otimes_{m=0}^{n-1} RX(x_m \pi)$. $D(x)$ loads the data point $\ket{x}$ using basis encoding and $DL(x; p_-, p_+)$
further creates the desired superposition from $\ket{x}$ based on the distribution $\mu_x$ as shown below,
$D(x)$ applies $RX(0) = \mathbb I $ to the $i^{th}$ qubit if $x_i = 0$ else apply $RX(\pi) = -\iota X$ if $x_i = 1$, therefore, $D(x)\ket{0} = (-\iota)^{\sum_{j=0}^{n-1} x_j}\ket{x}$.
Denote, $DL_M(x_M; p_-, p_+): = RY(\theta_0 + x_M \theta_D)$ and note that,
\begin{equation}
    \begin{split}
        DL_M(0; p_-, p_+) \ket{0} &= RY(\theta_0)\ket{0} \\
        &= \sqrt{1-p_+} \ket{0} + \sqrt{p_+} \ket{1} \\
        DL_M(1; p_-, p_+) \ket{1} &= RY(\theta_0 + \theta_D)\ket{1} \\
        &= \sqrt{p_-} \ket{0} + \sqrt{1 - p_-} \ket{1}.
    \end{split}
\end{equation} 
Therefore,
\begin{equation}
\label{eq:circ}
\begin{split}
     DL_M(x_M;p_{-}, p_{+})\ket{x_M} = (\sqrt{P_F^{(M)}} \ket{\bar{x}_M} + \sqrt{1-P_F^{(M)}} \ket{{x_M}})  
\end{split}    
\end{equation}
where, as discussed before $P_F^{(m)} := (p_+)^{1-x_m}(p_-)^{x_m}$. Using Equation \ref{eq:circ} and the fact that $P(i|x_m) = (\sqrt{P_F^{(m)}}^{\mathbb I(i \ne x_m)} \sqrt{1-P_F^{(m)}}^{\mathbb I(i = x_m)})$, 

\begin{equation}
    \begin{split}
        DL(x; p_{-}, p_{-}) \ket{x} &= \otimes_{m=0}^{n-1} (\sqrt{P_F^{(m)}} \ket{\bar{x}_m} + \sqrt{1-P_F^{(m)}} \ket{{x_m}})\\ 
        &= \otimes_{m=0}^{n-1} \sum_{i=0}^{1}(\sqrt{P(i|x_m)} \ket{i} \\
        &= \sum_{i_0, ..., i_{n-1} = 0}^{1} \Pi_{m=0}^{n-1} \sqrt{P(i_m|x_m)}\otimes_{m=0}^{n-1}  \ket{i_m} \\
        &= \sum_{i=0}^{2^n-1} \sqrt{P(i|x)} \ket{i} = \sum_{i=0}^{2^n-1} \sqrt{\mu_x(i)} \ket{i}
    \end{split}
\end{equation}
Therefore, as desired $DL(x; p_-, p_+) \ket{x} =  \ket{\psi(x)} := \sum_{i=0}^{2^n-1} \sqrt{\mu_x(i)} \ket{i}$ and $U(x, p_-, p_+) \ket{0} = (-\iota)^{\sum_{j=0}^{n-1}x_j} \ket{\psi(x)}$. 
Moreover, consider $ G = (-\iota \iota)^{\sum_{j=0}^{n-1} x_j} (2 \ket{\psi(x)}\bra{\psi(x)} - \mathbb I_{2^n})O = \mathcal{R}_{\ket{\psi}_x} O_f$. Hence, the global phase does not change the final Grover's operator.

\end{proof}

\subsubsection{Asymptotic analysis and Quadratic speedup}
\begin{proof}[Proof of Lemma \ref{lemma:convergence_quantum}]
Let $a$ denote the output of the exact smooth classifier and $\tilde a$ be the output of SQC(...). From the analysis of Quantum Amplitude Estimation by \cite{Brassard2000QuantumAA} we know the following results, for $t$ counting qubits and $T:=2^t$,

\begin{equation}
\label{eq:app:bound}
|\tilde a - a| \le \frac{2\pi \sqrt{a(1-a)}}{T} + \frac{ \pi^2}{T^2} \le \frac{\pi }{T} + \frac{ \pi^2}{T^2}
\end{equation}

with probability $p_0 \ge \frac{8}{\pi^2}$. Number of calls to the oracle is also $T$; $T-1$ calls in the controlled Grover's operation and $1$ in eigenstate preparation. Therefore if $\frac{\pi }{T} + \frac{ \pi^2}{T^2} < \epsilon \implies |a - \tilde a| < \epsilon$. Hence $T = O(\frac{1}{\epsilon})$ calls are required to get the correct estimate upto a bound of $\epsilon$ with a probability of $\frac{8}{\pi^2}$. 

\textbf{Using median estimate to boost the success probability:}

The motivation behind using the median estimate is that in order for the median to fail the experiment must fail more than half of the times therefore as the number of samples increase the probability of failure decreases exponentially.

Let $X$ denote the random variable corresponding to the success of the event that upon using $T$ evaluations of the base classifier in the phase estimation algorithm the error bound in Equation \ref{eq:app:bound} is achieved. Clearly $X$ is Bernoulli with $p_0 $ (i.e. $X \sim B(p_0)$). Let $Y = \sum_{i=1}^{N}X_i$, i.e. $Y \sim Binomial(N,p_0)$ then the probability that that the median fails to satisfy the bound is given by:
\begin{equation}
\label{eq:app:binom}
P[|\tilde a - a| > \frac{\pi}{T} + \frac{\pi^2}{T^2}] \le P[Y \le \frac{N}{2}]
\end{equation}

Using the Chernoff’s bound for binomial, 
\begin{equation}
P[Y \le (1-\delta)\mu] \le e^{-\frac{\delta^2\mu}{2}},    
\end{equation}
choose $\delta = 1 - \frac{1}{2p_0}$ to get a bound for Equation \ref{eq:app:binom}: 
\begin{equation}
P[|\tilde a - a| > \frac{\pi}{T} + \frac{\pi^2}{T^2}] \le P[Y \le \frac{N}{2}]
\le \exp(-\frac{(1-\frac{\pi^2}{16})^2N(\frac{8}{\pi^2}))}{2})
\end{equation}  

Therefore, choose $N = \lceil 17 \ln(\frac{1}{\delta}) \rceil$ to have a certainty of $1-\delta$  by which the median estimate is within the desired bounds. Overall, to find the estimate of $a$ within the tolerance of $\epsilon$  with a probability of $1-\alpha$ , the quantum phase estimation algorithm requires $O(\frac{1}{\epsilon}\ln(\frac{1}{\alpha}))$ calls to the base classifier.

\end{proof}

\begin{proof}[Proof of Theorem \ref{thm:main}]
For quadratic speedup, we just need to show that the classical method requires $O(\ln(\frac{1}{\delta})\frac{1}{\epsilon^2})$ calls to the oracle to get an estimate within a bound of $\epsilon$ with probability $1 - \delta$. This can be done using Theorem $2$ of \cite{BinomialConfidence}. For $N$ number of evaluations of the base classifier the expected length of the Clopper-Pearson interval with probability $1-\delta$ can be written as 
\begin{equation}
   \begin{split}
        \epsilon &= \mathbb E[L_{cp}]
        = 2z_{\delta/2}N^{-1/2}(p(1-p)^{1/2}) + O(N^{-1}) \\
        &= \Omega(\sqrt{\ln(\frac{1}{\delta})}N^{-1/2}) 
\end{split}
\end{equation}

where $z_{\alpha/2}$ is the upper $\alpha /2$ quantile of the normal distribution and the for the last asymptotic equality we use the asymptotic behaviour of the tail of the Gaussian. Therefore to obtain with a probability $1-\delta$, an estimate $\tilde g(x)$ of the smooth classifier $g(x)$ such that $|\tilde g(x) - g(x)| < \epsilon$ we need $O(\ln(\frac{1}{\delta})\frac{1}{\epsilon^2})$ calls to the classifier.

\end{proof}

\subsection{Further details on experiments}
\label{extra_experiments}
\subsubsection{Binary-MNIST Classification}
\label{app_bin_MNSIT}
To get the base classifier we train a convolution neural network~\cite{cnn_lecun} coupled with adversarial training to solve the multi-class classification task by minimizing the cross-entropy loss. For adversarial training, we sample 128 perturbed images \cite{bojchevski2023efficient} and take the empirical mean of all the perturbed predictions as an estimate of the smooth classifier from Equation \ref{eq:drs_smoothclass}. We choose the smoothing defined by $p_{+} = p_{-} = 0.2$ for training the network. Note that, although the certification works for any model, we still use adversarial training to ensure that the resulting smooth classifier performs well. It is not essential to use the same smoothing parameters to construct the smooth classifier. The data space is represented as $\mathcal X := \{0,1\}^{16 \times 16}$, and the window of perturbation is $\Pi :\{0, 1, ..., 16\} \mapsto \{0, 1, ..., 15\} \times \{0, 1, ..., 15\}$, i.e., $\Pi(i)$ gives the index of the $i^{th}$ pixel of the window. The set $\Pi$ here corresponds to the set $\mathcal I_N$ in \cref{discrete_perturbations}. We define the perturbation model $p(x_i, i)$ as $p(x_i, i) := \overline{x_i}$. i.e. independent of the location of the pixel $1$ can be perturbed to $0$ and vice-versa. This allows us to define the perturbation space with $2^{17}$ perturbations. Further, we define the data-dependent probability $q(x_i, i) := p_+ \mathbf{1}[x_i = 0] + p_- \mathbf{1}[x_i = 1]$. The characterization of $p(x_i, i)$ and $q(x_i, i)$ allows us to map this problem to the framework of \cref{discrete_perturbations}.

\subsubsection{Graphs Classification}
Here, we compare the convergence rates of classical and quantum smooth classifiers. For the graph classification task, if the quantum algorithm requires N samples, the classical algorithm requires $\mathcal O(N1.78)$ samples, as shown in \cref{fig:Convergence_Graphs}.

\begin{figure}[t]
\vskip 0.2in
\begin{center}
\centerline{\includegraphics[width=0.75\columnwidth]{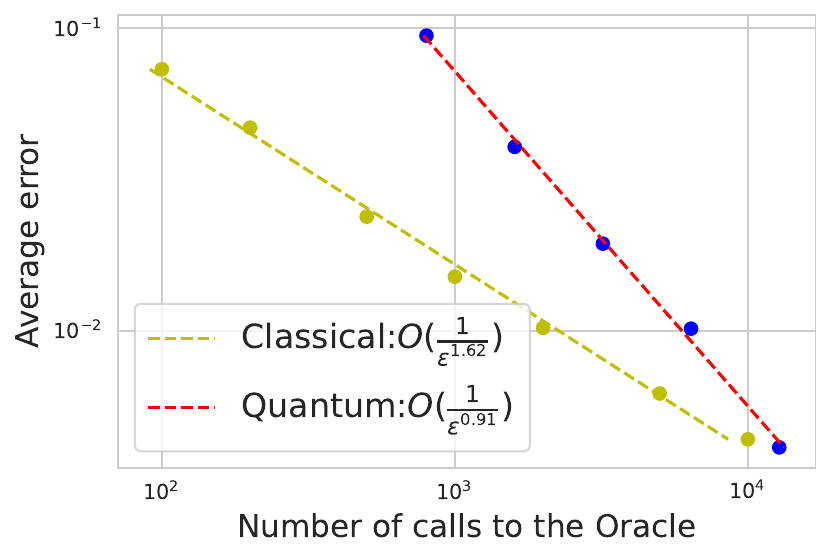}}
\caption{Convergence of error with number of calls to the oracle. Comparison for \textbf{Classical} vs \textbf{Quantum} algorithm with errors averaged over     $170$ graphs from the Graph classification test case.}
\label{fig:Convergence_Graphs}

\end{center}
\vskip -0.2in
\end{figure}

\subsubsection{Text Classification}\label{app:textclassification}
Let $\vect x$ denote the list of input tokens, and $\Phi$ the removal of a token. Further, $\mathcal I$ denotes the indices of the tokens corresponding to stop words. For $i \in \mathcal I $, define, $p(x_i, i) = \Phi$ and $q(x_i, i) = 0$ if $x_i = \Phi$, $q(x_i, i) = p_{rem}$ otherwise. Essentially, given $\vect x$, we map it to the binary vector of length $\lvert \mathcal I \rvert$ with all $0$'s and change $0$ to $1$ whenever a stop word is removed. For the smoothing distribution we allow $0$ to change to $1$ with $p_+ = p_{rem}$ and $1$ to $0$ with $p_- = 0$.  
We show the certified ratio for the exact and the quantum classifier with $p_{rem} = 0.5$ in \cref{fig:text_certificate_5}.
\begin{figure}[h!]
\vskip 0.2in
\begin{center}
\centerline{\includegraphics[width=0.75\columnwidth]{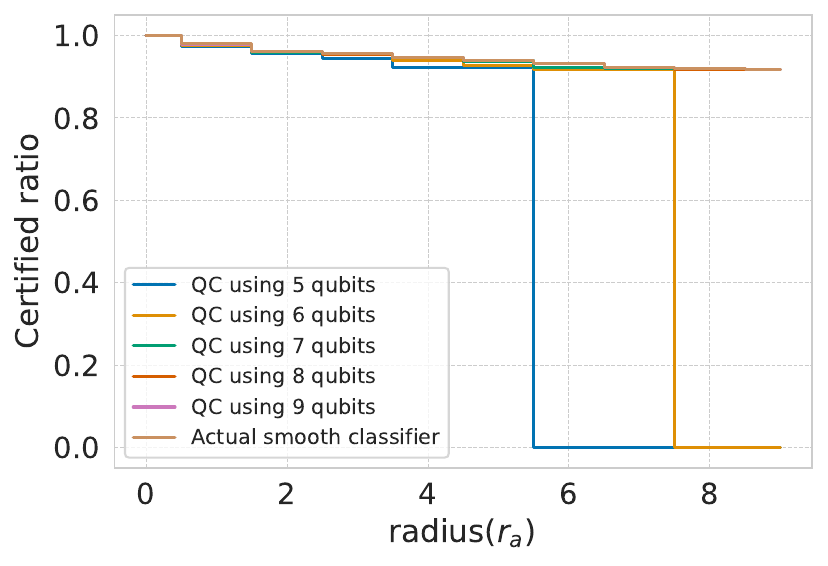}}
\caption{Certified ratio, exact classifier vs quantum classifier ($p_{\text{rem}} = 0.5$) using $5,6,7,8,9$ counting qubits evaluated for $600$ ($300$ for each polarity) randomly selected reviews with $8-10$ neutral words.}
\label{fig:text_certificate_5}
\end{center}
\vskip -0.2in
\end{figure}

\printbibliography

\end{document}